\documentclass[12pt]{amsart}

\usepackage[margin={1.5in, 1in}]{geometry}
\usepackage{amsmath,amssymb,amsfonts}
\usepackage{graphicx}
\usepackage{textcomp}
\def\BibTeX{{\rm B\kern-.05em{\sc i\kern-.025em b}\kern-.08em
    T\kern-.1667em\lower.7ex\hbox{E}\kern-.125emX}}

\usepackage{microtype}

\usepackage[T1]{fontenc}
\usepackage[utf8]{inputenc}
\usepackage{mathscinet}
\pdfoutput=1
\usepackage{amsmath,amsthm,wrapfig}
	\usepackage[utf8]{inputenc} 
\usepackage{amsmath, amssymb,bm, cases, mathtools, thmtools}
\usepackage{verbatim}
\usepackage{graphicx}\graphicspath{{figures/}}
\usepackage{multicol}
\usepackage{tabularx}
\usepackage[usenames,dvipsnames]{xcolor}
\usepackage{mathrsfs} 
\usepackage{caption}
\usepackage{algorithm}
\usepackage{algorithmicx}
\usepackage[noend]{algpseudocode}
\usepackage{array,booktabs,arydshln,xcolor}

\usepackage[%
		minnames=1,maxnames=99,maxcitenames=2,
		style=numeric-comp,%
		sorting=none,
		sortcites=false, %
		doi=false,url=false,
		uniquename=init,
		giveninits=true,
		hyperref,natbib,
		backend=biber]{biblatex}
\renewbibmacro{in:}{%
	\ifentrytype{article}{}{\printtext{\bibstring{in}\intitlepunct}}}

\usepackage[T1]{fontenc}
\usepackage{inconsolata}
\usepackage{dsfont}

\usepackage{listings} %

\usepackage{enumitem}
\usepackage{booktabs}       %
\usepackage{nicefrac}       %

\usepackage{lineno}
\usepackage{bbm}

\usepackage{caption}
\usepackage{subcaption}

\usepackage{datetime}

\DeclareMathAlphabet\EuRoman{U}{eur}{m}{n}
\SetMathAlphabet\EuRoman{bold}{U}{eur}{b}{n}

\declaretheorem[style=plain,numberwithin=section,name=Theorem]{theorem}

\declaretheorem[style=plain,sibling=theorem,name=Corollary]{corollary}

\declaretheorem[style=definition,sibling=theorem,name=Definition]{definition}

\declaretheorem[style=remark,qed=$\triangleleft$,sibling=theorem,name=Remark]{remark}
\numberwithin{theorem}{section}

\usepackage{xparse}
\usepackage{xstring}
\usepackage{xspace}

\def\[#1\]{\begin{align}#1\end{align}}
\def\*[#1\]{\begin{align*}#1\end{align*}}

\newcommand{\minf}[1]{I(#1)}

\newcommand{\entr}[1]{\mathrm{H}(#1)}

\newcommand{\centr}[2]{\mathrm{H}^{#1}(#2)}

\newcommand{\cPr}[2]{\Pr^{#1}[#2]}

\newcommand{\Dist}{\mathcal D}
\newcommand{\dataspace}{\mathcal Z}
\newcommand\optparen[1]{\ifthenelse{\equal{#1}{}}{}{(#1)}}
\newcommand{\RiskChar}{R}
\newcommand{\Risk}[2]{\RiskChar_{#1}\optparen{#2}}
\newcommand{\EmpRisk}[2]{\hat \RiskChar_{#1}\optparen{#2}}

\newcommand{\Naturals}{\mathbb{N}}

\newcommand{\Reals}{\mathbb{R}}

\newcommand{\as}{\textrm{a.s.}}

\newcommand{\dee}{\mathrm{d}}

\newcommand{\inspace}{\mathcal X}
\newcommand{\outspace}{\mathcal Y}

\DeclareMathOperator*{\newlim}{\mathrm{lim}\vphantom{\mathrm{infsup}}}

\DeclareMathOperator*{\newmax}{\mathrm{max}\vphantom{\mathrm{infsup}}}
\DeclareMathOperator*{\newinf}{\mathrm{inf}\vphantom{\mathrm{infsup}}}

\renewcommand{\lim}{\newlim}

\renewcommand{\max}{\newmax}
\renewcommand{\inf}{\newinf}

\newcommand{\ProbMeasures}[1]{\mathcal{M}_1(#1)}

\renewcommand{\Pr}{\mathbb{P}}
\def\EE{\mathbb{E}}

\newcommand{\defn}[1]{\textit{#1}}

\newcommand{\equaldist}{\overset{d}{=}}

\newcommand{\iid}{i.i.d.}

\newcommand{\KLname}{\mathrm{KL}}

\newcommand{\KL}[2]{\KLname(#1 \,\|\,#2)}

\newcommand{\XX}{\Reals^{\idim}}
\newcommand{\YY}{K}

\newcommand{\loss}{\ell}

\newcommand{\e}{\mathrm{e}}

\newcommand{\Alg}{\mathcal{A}}

\newcommand{\lcrx}[4][{-1}]{
	\IfEq{#1}{-1}{\left #2 {{{{#3}}}} \right #4}{
   	\IfEq{#1}{0}{#2 {{{{#3}}}} #4}{
	\IfEq{#1}{1}{\bigl #2 {{{{#3}}}} \bigr #4}{
	\IfEq{#1}{2}{\Bigl #2 {{{{#3}}}} \Bigr #4}{
	\IfEq{#1}{3}{\biggl #2 {{{{#3}}}} \biggr #4}{
	\IfEq{#1}{4}{\Biggl #2 {{{{#3}}}} \Biggr #4}{
    \GenericWarning{"4th argument to lcrx must be -1, 0, 1, 2, 3, or 4"}
    }}}}}}}
\newcommand{\rbra}[2][{-1}]{\lcrx[#1] ( {#2} ) }
\newcommand{\sbra}[2][{-1}]{\lcrx[#1] [ {#2} ] }

\newcommand{\rnderiv}[2]{\frac{\text{d} #1}{\text{d} #2}}

\newcommand{\cEE}[1]{\EE^{#1}}

\newcommand{\indep}{\mathrel{\perp\mkern-9mu\perp}}

\newcommand{\dminf}[2]{I^{#1} (#2)}

\newcommand{\indic}[1]{ \mathds{1}[#1]}

\newcommand{\supersam}{\tilde{Z}}

\newcommand{\range}[1]{ [#1] }

\newcommand{\vcd}{d} %

\newcommand{\targetfun}{h^\star}

\newcommand{\binaryentr}[1]{\mathrm{H}_{b}(#1)}

\newcommand{\loer}{\ensuremath{\hat R_{\mathrm{loo}}}\xspace}

\renewcommand{\defn}[1]{\emph{#1}}

\newcommand{\EGE}{\ensuremath{\mathrm{EGE}_{\Dist}(\Alg)}}

\newcommand{\loecmi}{LOO${}^e$CMI\xspace}
\newcommand{\lecmi}{\ensuremath{\mathrm{LOO}{}^e\mathrm{CMI}_{\Dist}(\Alg)}\xspace}

\newcommand{\lossvec}{L}
\newcommand{\lossreal}[1]{0_{(#1)}}

\usepackage[hide=true,setmargin=true,marginparwidth=1.25in]{marginalia}

 \usepackage[colorlinks,citecolor=blue,urlcolor=blue,linkcolor=RawSienna]{hyperref}

\usepackage[capitalize]{cleveref}

\crefname{lemma}{Lemma}{Lemmas}
\crefname{corollary}{Corollary}{Corollaries}
\crefname{theorem}{Theorem}{Theorems}

\crefname{problem}{Conjecture}{Conjectures}

\crefname{statement}{Statement}{Statement}

\newcommand{\HS}{\mathcal H}    
\begin{document}
\setlength{\marginparsep}{0.05in}

\title[Leave-One-Out Conditional Mutual Information]
{Understanding Generalization via Leave-One-Out Conditional Mutual Information}

\author[M. Haghifam]{Mahdi Haghifam$^{1,2}$}
\author[S. Moran]{Shay Moran$^{3,4}$}
\author[D. M. Roy]{Daniel M. Roy$^{1,2}$}
\author[G. K. Dziugaite]{Gintare Karolina Dziugaite$^{4,5}$}

\thanks{$^1$ University of Toronto\ \ $^2$ Vector Institute\ \ $^3$ Technion\ \ $^4$ Google Research \ \ $^5$ Mila; McGill. 
This work was presented at the
2022 IEEE Int.\ Symposium on Information Theory \cite{haghifam2022isit}.}

\maketitle

\begin{abstract}
We study the mutual information between (certain summaries of) the output of a learning algorithm and its $n$ training data, conditional on a supersample of $n+1$ i.i.d.\ data from which the training data is chosen at random without replacement. These leave-one-out variants of the conditional mutual information (CMI) of an algorithm (Steinke and Zakynthinou, 2020) are also seen to control the mean generalization error of learning algorithms with bounded loss functions. For learning algorithms achieving zero empirical risk under 0--1 loss (i.e., interpolating algorithms), we provide an explicit connection between leave-one-out CMI and the classical leave-one-out error estimate of the risk. Using this connection, we obtain upper and lower bounds on risk in terms of the (evaluated) leave-one-out CMI. When the limiting risk is constant or decays polynomially, the bounds converge to within a constant factor of two. As an application, we analyze the population risk of the one-inclusion graph algorithm, a general-purpose transductive learning algorithm for VC classes in the realizable setting. Using leave-one-out CMI, we match the optimal bound for learning VC classes in the realizable setting, answering an open challenge raised by Steinke and Zakynthinou (2020). Finally, in order to understand the role of leave-one-out CMI in studying generalization, we place leave-one-out CMI in a hierarchy of measures, with a novel unconditional mutual information at the root. For 0--1 loss and interpolating learning algorithms, this mutual information is observed to be precisely the risk.
\end{abstract}

\renewcommand{\XX}{\mathcal X}
\renewcommand{\YY}{\mathcal Y}
\renewcommand{\SS}{S_n}
\renewcommand{\Alg}{\mathcal A_n}
\newcommand{\TheAlg}{\mathcal A}
\newcommand{\SSS}{\inspace}
\newcommand{\TTT}{\mathcal {T}}

\section{Introduction}

In this paper, we study 
generalization in supervised learning. Formally, consider spaces of inputs $\inspace$, labels $\outspace$, and classifiers $\HS\subseteq \outspace^ \inspace$.
Let $\Dist$ be a distribution on $\dataspace =   \inspace \times \outspace  $. The \emph{empirical risk} of a classifier $h: \inspace \to \outspace$ on a sample $s=((x_1,y_1),\dots,(x_n,y_n)) \in \dataspace^n$
 is $\EmpRisk{s}{h} = n^{-1} \sum_{i \in [n]} \loss(h,(x_i,y_i))$, where
$\loss: \outspace^ \inspace \times \dataspace \to \Reals_{+}$ is the loss function.  Let $\SS \sim \Dist^n$, i.e., let $\SS$ be a sequence of i.i.d.\ random elements in $\dataspace$ with common distribution $\Dist$. The \emph{risk} of $h$ is $\Risk{\Dist}{h} = \EE\EmpRisk{\SS}{h}$, 
 where $\EE$ denotes the expectation operator. 
 A distribution $\Dist$ is said to be \defn{realizable by a class $\HS \subseteq \YY^\XX$} if $\inf_{h \in \HS}\Risk{\Dist}{h} = 0$. 
 Let $\TheAlg = (\Alg)_{n\ge 1}$ denote a (potentially randomized) learning algorithm, which, for any positive integer $n$, maps $\SS$ to an element $\Alg(\SS)$ of $\outspace^\inspace$. 
 We say $\Alg$ is an interpolating algorithm for datasets of size $n$  if $\EmpRisk{\SS}{\Alg(\SS)}=0$ a.s. If $\TheAlg$ is interpolating for all $n$, we say $\TheAlg$ is  consistent.  
Our primary interest in this paper is the \defn{expected generalization error} of $\Alg$ with respect to $\Dist$, defined as
\*[
 \EGE = \EE \sbra[0]{\Risk{\Dist}{\Alg(\SS)}- \EmpRisk{\SS}{\Alg(\SS)} }   ,
\]
where we average over both the choice of training sample and the randomness within the algorithm $\Alg$. Characterizing $\EGE$ provides us a handle on the performance of the learned classifier on unseen data. Note that, for interpolating algorithms, expected generalization error is expected risk, i.e., $\EGE = \EE \sbra[0]{\Risk{\Dist}{\Alg(\SS)}}\triangleq \Risk{\Dist}{\Alg}$.

Among the many contributions of learning theory are various frameworks for analyzing the generalization error of learning algorithms. Examples include uniform convergence and VC theory \citep{vapnik1974theory}, stability \citep{bousquet2002stability}, and
differential privacy \citep{dwork2015generalization}.
In recent years, there has been a flurry of interest in the use of information-theoretic quantities for characterizing the expected generalization error of learning algorithms. This line of work, initiated by \citet{RussoZou16} and \citet{XuRaginsky2017},  has been extended in many directions. The basic result is that the generalization error can be bounded in terms of the mutual information between the training data and the learned classifier. Unfortunately, this framework is provably unable to explain generalization in some settings \citep{bassily2018learners,nachum2018direct,roishay,BuZouVeeravalli2019}. For instance, \citet{roishay} show that, for the class of one-dimensional thresholds\footnote{
This concept class can be defined as follows. Let $\mathcal{X}=\{1,...,m\}$. 	
The class of one-dimensional thresholds over $\{1,...,m\}$ is  $\mathcal{H}_m=\{h_k\vert k\in \mathbb{N}\}$, where $h_k: \mathcal{X} \to \{0,1\}$ is $h_k(x)=\mathds{1}[x> k]$.}
 over $\{1,\dots,m\}$, $m \in \Naturals$, for \emph{every} learning algorithm $\TheAlg$, there exists a realizable distribution such that either the risk (population loss) is ``large'' or the mutual information scales as  $\Omega(m)$, leading to vacuous generalization bounds. However, for the same problem, VC theory implies that the population risk of \emph{every} empirical risk minimizer goes to zero at the rate of $O(1/n)$.

In order to overcome some of these limitations, \citet{steinke2020reasoning} propose a modified information-theoretic framework, which they dub \emph{Conditional Mutual Information} (CMI).
In this framework, a ``supersample'' is introduced and the training set is obtained as a subsample.
In particular, the supersample is
a $2 \times n$ array, $\supersam \in \dataspace^{2\times n}$, composed of  \iid~samples from the data distribution $\Dist$.
Then,
for an array $\tilde{U}$ independent from $\supersam$ and uniformly distributed in $\{0,1\}^n$ (hence, $\tilde{U}$ is composed of \iid\ Bernoulli random variables with mean $1/2$), 
the training set 
is
$\SS=(\supersam_{\tilde{U}_i,i})_{i \in [n]}$.
The conditional mutual information (CMI) of a learning algorithm $\Alg$ on $n$ training data from $\Dist$ is defined to be 
$\minf{\Alg(\SS);\tilde U\vert \supersam}$,
i.e., the conditional mutual information of the output classifier $\Alg(\SS)$ and $\tilde{U}$, conditioned on $\supersam$. 
In the same work, 
the authors also introduce the notion of the 
\emph{evaluated} CMI of an algorithm, $\minf{L;\tilde U\vert \supersam}$,
where $L$ is the $2 \times n$ array of losses incurred by $\Alg(\SS)$ on the supersample $\supersam$.
Results by \citet{steinke2020reasoning,haghifam2021towards,harutyunyan2021information} show that CMI and its evaluated CMI (eCMI) counterpart provide an expressive framework for bounding generalization,
encompassing multiple existing frameworks for proving generalization. 
Despite these results, it is not known whether the CMI/eCMI framework can be used to characterize minimax rates of expected risk in all circumstances. 
As one key example, it is not known whether any information-theoretic framework can obtain minimax rates for supervised classification problems with VC classes on realizable distributions.

\subsection{Our Results}
In this paper, we introduce a new information-theoretic measure of dependence between the output of a learning algorithm and its input,
based on a leave-one-out analogue of eCMI:
\begin{definition}
\label{def:loocmi}
Let $\TheAlg$ be a learning algorithm. 
Let $n \in \Naturals$, let $\supersam=(\supersam_{i})_{i \in  [n+1]}$ be an $n+1$-array of i.i.d.\ random elements in the dataspace $\dataspace$ with common distribution $\Dist$. Let $U$ be a random variables distributed uniformly on $[n+1]$, independent from $\supersam$.
For $u \in [n+1]$, let $\supersam_{-u}$ denote $(\supersam_{j})_{j\in [n+1],j\neq u}$, i.e., the supersample with the $u$'th element removed. Define $\SS=\supersam_{-U}$.
Let $L\in \Reals_{+}^{n+1}$ be the array with entries $L_{i}=\loss(\Alg(\SS),\supersam_{i})$ for $i\in \range{n+1}$.  The \defn{leave-one-out (evaluated) conditional mutual information of $\Alg$ with respect to $\Dist$} is 
\*[
\lecmi \triangleq \minf{L;U\vert\supersam}.
\]
\end{definition}
Our notion is inspired by the \emph{leave-one-out error} estimator of the generalization error \citep{mohri2018foundations},  a widely used surrogate for the expected generalization error. Intuitively, \lecmi measures how well one can ``identify'' which point from the supersample is being held-out, given the supersample and the losses incurred on each of its elements. 
In \cref{sec:gen-bounds}, we show that bounded \lecmi implies generalization in both interpolating and agnostic learning algorithms. Our generalization bound for the interpolating case enjoys the property that it is never greater than one. 

\NA{
In \cref{sec:connect-leaveout}, in the context of 0--1 loss and interpolating learning algorithms, 
we establish a general connection between \loecmi and the classical leave-one-out error estimator of the risk.
Using this result, we show that \lecmi precisely characterizes the risk of interpolating learning algorithms in many common situations. Specifically, we show that, for every data distribution and interpolating learner, the \lecmi framework yields a risk bound that vanishes  \emph{if and only if} the risk also vanishes as the  number training samples diverges. Also, the \lecmi framework is the first information-theoretic framework that can be shown to determine the risk of any consistent learning algorithms whose risk either converge to a non-zero value or converges to zero polynomially with the number of samples. 
}

In \cref{sec:appl}, as an application of our general connection with leave-one-out error analysis, we characterize the \loecmi of the one-inclusion graph algorithm, which was introduced by \citet{haussler1994predicting} for learning Vapnik--Chervonenkis (VC) classes. Using our framework, we obtain the \emph{optimal} risk bound for every VC class in the realizable setting. 
In doing so, we answer the open problem stated in \citep{steinke2020openproblem} of characterizing the expected excess risk of learning VC classes using an
information-theoretic framework. 

In \cref{sec:grun}, we consider several additional measures of information based on the supersample structure introduced in \cref{def:loocmi},
show that they all control generalization,
and discuss their inter-relationships.
In particular, we present the chain of inequalities
\[
\minf{L;U} 
\leq \minf{L;U\vert \supersam} 
\leq \minf{\hat{Y};U\vert \supersam} 
\leq \minf{\Alg(\SS);U\vert \supersam} 
\leq \minf{\Alg(\SS);\SS}, \label{eq:thechain}
\]
where $\minf{\Alg(\SS);U\vert \supersam}$ is the non-evaluated analogue of our notion and $\hat Y$ is the length-$n+1$ list of labels predicted by $\Alg(\SS)$ on the inputs in the supersample $\supersam$.
With the exception of the first quantity, $\minf{L;U}$, all of these (or close analogues) have been studied in the literature. In the special case of binary classification by an interpolating classifier, $\minf{L;U}$ is precisely the risk, yielding a simple argument for why the other notions also bound risk (equivalently, generalization error) in this setting. 
Based on results presented herein and elsewhere, we discuss gaps between these various notions in \cref{eq:thechain} and the roles they can play in understanding generalization. 
In particular, we show that \loecmi is the weakest measure in this chain that can  characterize the risk for every interpolating algorithm under 0--1 loss.

\color{black}

\subsection{Related Work}

For supervised learning algorithms, \citet[Sec.~6]{steinke2020reasoning} and \citet{harutyunyan2021information}  define information-theoretic measures of dependence based on the losses and predictions, respectively, of a learning algorithm rather than the learned classifier. 
The results in both of these papers are based on the CMI framework, i.e., a supersample with $2n$ samples. 
Neither of these papers recover  optimal bounds for learning VC classes in the realizable setting. 
\citet{CCMI20} combine chaining with CMI to study generalization of deterministic learning algorithms.
For ERM on classes of finite VC dimension $d$ in the \emph{agnostic} case, i.e., $\inf_{h\in \HS}\Risk{\Dist}{h}\neq 0$,
\citeauthor{CCMI20} use chaining CMI to obtain bounds on the expected generalization error that achieve the optimal rate $O(\sqrt{d/n})$. 
See also \citep{aminian2021jensen,clerico2022chained,zhou2020individually,hafez2022information, IbrahimEspositoGastpar19,JiaoHanWeissman17,asadi2018chaining,zhou2022stochastic,rodriguez2020random}.

\subsection{Notation}
Let $P,Q$ be probability measures. %
For a $P$-integrable  
function $f$, let $\smash{P[f] = \int f \dee P}$.
When $Q$ is absolutely continuous with respect to $P$, denoted $Q \ll P$,
write $\rnderiv{Q}{P}$ for (an arbitrary version of) the 
Radon--Nikodym derivative (or density) of $Q$ with respect to $P$. 
The \defn{KL divergence} (or \defn{relative entropy}) \defn{of $Q$ with respect to $P$},
denoted $\KL{Q}{P}$, is defined as $Q[ \log \rnderiv{Q}{P} ]$ when $Q \ll P$ and
infinity otherwise.

For a random element $X$ in some measurable space $\SSS$,
let $\Pr[X]$ denote its distribution, which lives in the space
$\ProbMeasures{\SSS}$ of all probability measures on $\SSS$.
Given another random element, say $Y$ in $\TTT$,
let
$\cPr{Y}{X}$ denote the conditional
distribution of $X$ given $Y$.
If $X$ and $Y$ are independent, denoted by $X\indep Y$, we have $\cPr{Y}{X} = \Pr[X]$ a.s. 

The \defn{mutual information between $X$ and $Y$} 
is $
\minf{X;Y} = \KL{\Pr[(X,Y)]} { \Pr[X] \otimes \Pr[Y]},$
where $\otimes$ forms the product measure. 
Writing $\cPr{Z}{(X,Y)}$ for the conditional distribution of the pair $(X,Y)$
given a random element $Z$,
the \defn{disintegrated mutual information between $X$ and $Y$ given $Z$,} is $
 \dminf{Z}{X;Y} =  \KL{\cPr{Z}{(X,Y)}}{\cPr{Z}{X} \otimes \cPr{Z}{Y} }$, and the conditional mutual information is $
 \minf{X;Y\vert Z}=\EE{[\dminf{Z}{X;Y}]}$. Similarly, we can define \defn{disintegrated entropy of $X$ given $Y$} denoted by $\centr{Y}{X} $, and its expectation gives the \defn{conditional entropy of $X$ given $Y$}, i.e.,  $\entr{X\vert Y} =  \EE[\centr{Y}{X}]$.

\section{Generalization Bounds}
\label{sec:gen-bounds}
In this section, we show that bounded \loecmi implies generalization. 
First, we provide a generalization bound for interpolating learning algorithms:
\begin{theorem}
\label{thm:genbound-consistent}
Let $n \in \Naturals$, assume loss is bounded in $[0,1]$, and let $\Alg$ be an interpolating learning algorithm. 
Then,
\begin{equation*}
   \Risk{\Dist}{\Alg} \leq \frac{\lecmi}{\log (n+1)}. 
\end{equation*}
\end{theorem}
\begin{remark}
By the definition of mutual information, $\lecmi\leq \entr{U\vert \supersam} \leq \log(n+1)$, since $U \indep \supersam$ and  $U$  is distributed uniformly on $[n+1]$. Therefore, the bound above is never greater than one.
\end{remark}
\begin{proof}[Proof of \cref{thm:genbound-consistent}]
Let $\Dist$, $\supersam$, $U$, and $L$ be defined as in  \cref{def:loocmi}. Introduce $\tilde U$ such that  $\tilde{U}\equaldist U$ and $\tilde{U} \indep (\supersam,L)$. By the Donsker--Varadhan variational formula \citep[Prop.~4.15]{boucheron2013concentration}, for all bounded measurable functions $h$ 
and
$\lambda\in \Reals$,
\[
\begin{split}
&\minf{\supersam,L;U} = \KL{ \Pr{(\supersam,U,L})}{\Pr{({\supersam},{L})}\otimes\Pr{(\tilde{U})}} \\
&\geq \Pr{(\supersam,U,L)}(\lambda h) - \log \big[ \big(\Pr{({\supersam},{L})}\otimes\Pr{(\tilde{U})}\big)(\exp (\lambda h))\big]\label{eq:dv-lemma2}.
\end{split}
\]
Let $\alpha \in \Reals_{+}$ be a constant. 
Consider the function $h_\alpha: \dataspace^{n+1} \times [n+1] \times [0,1]^{n+1} \to \Reals$ given by
$h_\alpha(\tilde{z},u,l) = l_{u} - \alpha \sum_{i\in [n+1],i\neq u}l_i$.
Then
\[
\begin{split}
&\Pr{(\supersam,U,L)}(h_\alpha) = \EE[ \cEE{U}{h_\alpha(U,L,\supersam)}] \\
&=\EE\big[\EE^{U}\big[ \loss(\Alg(\supersam_{-U}),\supersam_U) - \alpha\smashoperator{\sum_{i\in[n],i\neq U}} \loss(\Alg(\supersam_{-U}),\supersam_i)\big]\big]\\
&\stackrel{(a)}{=} \EE\big[\cEE{U}{\big[ \loss(\Alg(\supersam_{-U}),\supersam_U)}\big]\\
&= \EE \big[ \Risk{\Dist}{\Alg(\SS)}\big] \label{eq:dv-egeterm2}.
\end{split}
\]
Here, $(a)$ follows from the fact that $\Alg$ is an interpolating algorithm, hence, for all $i\neq U$, $\loss(\Alg(\supersam_{-U}),\supersam_i)=0$ a.s. 
Thus, by \cref{eq:dv-lemma2}, \cref{eq:dv-egeterm2}, and $I(L,\supersam;U)=\lecmi$, we obtain
\[
\begin{split}
\label{eq:dv-simplified}
\Risk{\Dist}{\Alg} &\leq \frac{\lecmi}{\lambda}  \\
&+\frac{\log \big[\Pr{({\supersam},{L})}\otimes\Pr{(\tilde{U})}(\exp (\lambda h_\alpha))\big]}{\lambda}.
\end{split}
\]
To simplify the notation, let  $\loss(i,j)=\loss(\Alg(\supersam_{-i}),\supersam_{j})$ for $i$ and $j$ in $[n+1]$. We have
\[
\begin{split}
\label{eq:term2-dv-simple1}
&\Pr{(\supersam,L)}\otimes\Pr{(\tilde{U})}(\exp (\lambda h_\alpha))\\
&\stackrel{(a)}{=}\EE \big[ \exp\big[\lambda(\loss(U,\tilde{U})-\alpha \smashoperator{\sum_{i\in [n+1],i\neq \tilde{U}}} \loss(U,i))\big]\big]\\
&= \EE \big[ \indic{U=\tilde{U}} \exp\big[\lambda(\loss(U,\tilde{U})-\alpha \smashoperator{\sum_{i\in [n+1],i\neq \tilde{U}}} \loss(U,i))\big]\\
&\qquad +\EE \big[\indic{U\neq\tilde{U}} \exp\big[\lambda(\loss(U,\tilde{U})-\alpha \smashoperator{\sum_{i\in [n+1],i\neq \tilde{U}}} \loss(U,i))\big]\\
&\stackrel{(b)}{=}\EE \big[ \indic{U=\tilde{U}}\exp(\lambda\loss(U,U)) + \indic{U\neq\tilde{U}}\exp(-\lambda \alpha\loss(U,U))\big]\\
&\stackrel{(c)}{=}\EE \big[ \frac{1}{n+1} \exp(\lambda\loss(U,U)) + \frac{n}{n+1} \exp(-\lambda \alpha \loss(U,U))\big].
\end{split}
 \]
The equality $(a)$ follows from the chain rule. 
In particular, the expectation over $\supersam$ and $L$ can be written as $\EE \cEE{\supersam,U}$. Step $(b)$ follows from the fact that for all $i\in [n+1]$ and $i\neq U$, $\loss(U,i)=0$ \as,   since the algorithm is interpolating. In $(c)$ we take the expectation with respect to $\tilde{U}$ and use the fact that $\tilde{U}$ is independent of $\supersam$, $U$, and the internal randomness of $\TheAlg$. 
Let $\lambda = \log(n+1)$. 
Consider the function $f: [0,1]\to \Reals$, where $f(x)=\frac{1}{n+1} \exp(\lambda x) + \frac{n}{n+1} \exp(-\lambda \alpha x)$. Note that $f$ is the sum of two convex functions, defined on a bounded domain. It achieves its maximum over the endpoint, i.e., $x\in \{0,1\}$. Considering this observation,  we can further upper bound \cref{eq:term2-dv-simple1} by considering the following two cases: If $\loss(U,U)=1$, then
$\frac{\exp(\lambda \loss(U,U)) + n \exp(-\lambda \alpha \loss(U,U))}{n+1} =  1 + \frac{n}{n+1}\exp(-\alpha\log(n+1))$. Otherwise, in the case that $\loss(U,U)=0$, we have 
\*[
 \frac{\exp(\lambda \loss(U,U)) + n \exp(-\lambda \alpha \loss(U,U))}{n+1}=1.
\]
Therefore,  we conclude 
\[
\label{eq:dvterm2-finalterm}
\Pr{({\supersam},{L})}\otimes\Pr{(\tilde{U})}(\exp (\lambda h_\alpha)) &\leq 1 + \frac{n}{n+1}\exp(-\alpha\log(n+1)).
\]
By \cref{eq:dv-simplified} and \cref{eq:dvterm2-finalterm}
\[
\label{eq:prefinalterm}
\Risk{\Dist}{\Alg} \leq \frac{\minf{L;U\vert \supersam } + \log \rbra[1]{1+\frac{n \exp(-\alpha \log(n+1))}{n+1}}}{\log(n+1)}.
\]
Finally, letting $\alpha \to +\infty$ in \cref{eq:prefinalterm} concludes the proof. 
\end{proof}

\begin{remark}
 The special case of zero--one loss can be proven in a more direct way. See \cref{sec:grun}.
\end{remark}
\begin{remark}
In this remark, we consider a seemingly more elementary argument to establish \cref{thm:genbound-consistent}, but show 
that the logic underlying this approach is flawed. Assume that the loss function is zero--one loss and that the algorithm is interpolating. In this setup, note that $\centr{L,\supersam}{U}=0$ almost surely on the event $L\neq (0,\dots,0)$. The reason is that the algorithm can make an error only on the test point in $\supersam$. Therefore, we can write $\minf{L;U\vert \supersam } = \entr{U \vert \supersam}-\entr{U \vert L,\supersam}
= \log(n+1) - \EE{\sbra{\centr{L,\supersam}{U } \indic{L= (0,\dots,0)}}}$. 
One's intuition might suggest that,
conditional on the event $L= (0,\dots,0)$, 
there is no mutual information between $U$ and  $(\supersam,L)$.
One might then leap to   $\EE{\sbra{\centr{L,\supersam}{U }\vert L= (0,\dots,0)}}=\log(n+1)$, from which one would reason that $\minf{L;U\vert \supersam } = \log(n+1) - \log(n+1) (1- \Risk{\Dist}{\Alg})=\Risk{\Dist}{\Alg}\log(n+1)$.

However, the intuition is flawed.
Consider the class of thresholds in one dimension. In \cref{fig:thresh}, we consider learning from $n=4$ data, when the supersample is in a neighborhood of a sequence $(z_1,\dots,z_5)$. Given the training data $\SS$, let $x^\star=\max \{x| (x,1)\in \SS\}$ if $\SS$ contains at least one point with label $1$, otherwise let $x^\star = -\infty$. Consider the learning algorithm $\Alg(\SS)=\hat{h}$ where $\hat{h}(x) = \indic{x\leq x^\star}$. On the event $L=(0,\dots,0)$ and $\supersam$ is a small perturbation of the points $(z_1,\dots,z_5)$, we have $\cPr{L,\supersam}{U=3}=0$. 
The reason is, had it been the case that $U=3$, then the learning algorithm would have erred on its prediction for (the point corresponding to) $z_3$. Therefore, $\EE{\sbra{\centr{L,\supersam}{U }\vert L= (0,\dots,0)}}\neq\log(n+1)$, in general.
\end{remark}

\begin{figure}
    \centering
    \includegraphics[scale=0.4]{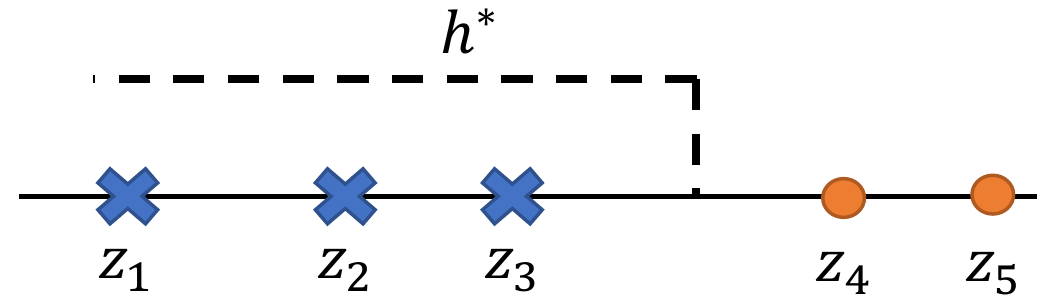}
    \caption{A counter example.}
    \label{fig:thresh}
\end{figure}

For an arbitrary learning algorithm, \loecmi still controls generalization. The proof of
the following theorem is deferred to \cref{sec:grun}.

\begin{theorem}
\label{thm:gen-agnostic}
Let $n \in \Naturals$. Assuming only that loss is bounded in $[0,1]$,
\begin{equation*}
    \EGE \leq \sqrt{2\lecmi}.
\end{equation*}
\end{theorem}
\section{A Connection with leave-one-out error}
\label{sec:connect-leaveout}

In this section, we describe a connection between \loecmi and the \emph{leave-one-out error}, a well-studied statistical estimator of risk
\citep{cover1969learning,stone1976cross}.
Using the supersample $\supersam$, 
the random variable
\[
\nonumber
\loer = \frac{1}{n+1}\sum_{i=1}^{n+1} \cEE{\supersam}{\sbra{\loss(\Alg({\supersam_{-i}}),Z_i)}}
\]
is a leave-one-out error estimate for the risk of $\Alg$, where we have averaged out the internal randomness in $\Alg$. (Note that, for deterministic learning algorithms, this averaging has no effect.) 
In order to connect this quantity to \loecmi, we first note that we can bound \lecmi in terms of the disintegrated entropy:
\[
\lecmi = \EE\sbra{\centr{\supersam}{L} - \centr{\supersam,U}{L} } \leq \EE\sbra{\centr{\supersam}{L}},
\]
where the second inequality is an equality for deterministic learning algorithms.
Let $\binaryentr{\cdot}$ denote the binary entropy function.
\begin{theorem}
\label{thm:leave-one-out}
Let $\Alg$ be an interpolating learning algorithm for some data distribution $\Dist$. 
Assume loss lies in $\{0,1\}$. Then, almost surely,
\[
\centr{\supersam}{L} \leq \binaryentr{\loer}+ \loer \log(n+1).
\]
\end{theorem}
\begin{proof}
Let $\lossreal{0} = (0,\dots,0) \in \{0,1\}^{n+1}$ and, for $i \in \{1,\dots,n+1\}$, let $\lossreal{i} \in \{0,1\}^{n+1}$ be equivalent to $\lossreal{0}$ but for a 1 at index $i$.
Due to interpolation, the support of $\lossvec$ is $\{\lossreal{i}\vert i \in \{0,\dots,n+1\} \}$. For each $i,j \in \{1,\dots,n+1\}$, let 
\*[
\kappa_{i,j}
&=\cPr{\supersam,U=j}{L=\lossreal{i}}=\cPr{\supersam}{\loss(\Alg(\supersam_{-j}),\supersam_i)=1}.
\]
Since $\Alg$ is a consistent algorithm $\kappa_{i,j}=0$ \as \ for $i\neq j$.
 Also, $\EE[\kappa_{i,i}]=\Risk{\Dist}{\Alg}$. 
 For $i \in \{1,\dots,n+1\}$,
\*[
\cPr{\supersam}{L=\lossreal{i}} &=\cEE{\supersam}{\big[\cPr{\supersam,U}{L=\lossreal{i}}\big]}\\
&= \frac{1}{n+1} \sum_{j=1}^{n+1} \kappa_{i,j}\\
&=\frac{\kappa_{i,i}}{n+1} ,
\]
where the last line follows since $\kappa_{i,j}=0$ for  $i\neq j$.
Note that the leave-one-out-error satisfies
$\loer = (n+1)^{-1}\sum_{i=1}^{n+1} \kappa_{i,i}$.
Then, 
by the definition of the entropy, 
\[
\centr{\supersam}{L} 
&=  -\sum_{i=0}^{n+1} \cPr{\supersam}{L=\lossreal{i}} \log \cPr{\supersam}{L=\lossreal{i}} \nonumber\\
&= -\big(1-\loer \big) \log \big(1- \loer \big)
- \sum_{i=1}^{n+1} \frac{\kappa_{i,i}}{n+1} \log \big( \frac{\kappa_{i,i}}{n+1}\big).
\label{eq:defentr}
\]

We now invoke the log-sum inequality for non-negative sequences $\{a_i\}_{i\in [n]}$ and $\{b_i\}_{i\in [n]}$, wherein $\sum_{i=1}^{n} a_i \log \frac{a_i}{b_i}\geq (\sum_{i\in [n]}a_i) \log \frac{\sum_{i\in [n]}a_i}{\sum_{i\in [n]}b_i} $. Using this inequality, we obtain
\[
\sum_{i=1}^{n+1} \frac{\kappa_{i,i}}{n+1} \log \Big( \frac{\kappa_{i,i}}{n+1}\Big)
\geq  \loer  \log \Big( \frac {\loer }{n+1} \Big) .
\label{eq:logsumineq}
\]
Therefore, by \cref{eq:defentr,eq:logsumineq},
\[
\centr{\supersam}{L} &\leq -\big(1-\loer \big) \log \big(1-\loer \big) \nonumber\\
 &\qquad -\loer \log \loer + \loer \log(n+1) \nonumber\\
 & = \binaryentr{\loer  }+ \loer \log(n+1), \nonumber
\]
as was to be shown.
\end{proof}
As a corollary, we provide an explicit bound on \loecmi.
\begin{corollary}
\label{thm:loo-ecmi}
Let $\Alg$ be a consistent learning algorithm for zero--one valued loss, 
 let $\Dist$ be a distribution on $\dataspace$, and assume that, with probability one, $\loer \leq \frac{\theta}{n+1}$, where $\theta$ is some $\supersam$-measurable random variable in $\Reals_+$.
Then, almost surely,
\*[
\dminf{\supersam}{L;U} \leq   
\begin{cases}
  1+ \frac{\theta \log(n+1)}{n+1},   &  \text{if $\frac{\theta}{n+1}\geq \frac{1}{2}$,} \\
 \frac{2\theta \log(n+1)}{n+1} + \frac{\theta+\exp(-1)}{n+1},  & \text{otherwise.}
\end{cases}
\]
\end{corollary}
\begin{proof}
For the case $2\theta \geq n+1$, upper-bounding the $\binaryentr{\loer}$ by one, we obtain the result. For the case $2\theta < n+1$, note that $\binaryentr{\loer} \leq \binaryentr{\frac{\theta}{n+1}}$. Using the well-known inequality $\binaryentr{x}\leq -x\log(x)+x$, we obtain 
\[
\nonumber
\dminf{\supersam}{L;U} &\leq \frac{\theta}{n+1}\log \frac{n+1}{\theta} + \frac{\theta}{n+1} + \frac{\theta}{n+1} \log(n+1)\\
&= \frac{\theta}{n+1}(\log(n+1) - \log(\theta))  +\frac{\theta}{n+1}  +\frac{\theta \log(n+1)}{n+1} \nonumber\\
&=\frac{2\theta \log(n+1)}{n+1} - \frac{\theta \log(\theta)}{n+1} +\frac{\theta}{n+1}  \nonumber\\
&\leq  \frac{2\theta \log(n+1)}{n+1} + \frac{\exp(-1)}{n+1} +\frac{\theta}{n+1} , \nonumber
\]
where the last line follows from $\max_{x>0} -x\log(x)=\exp(-1)$.
\end{proof}

 \cref{thm:loo-ecmi} can be used to obtain risk bound for a variety of consistent learning algorithms. For example, the Support Vector Machine (SVM) algorithm has a leave-one-out error guarantee in the sense of \cref{thm:loo-ecmi} for realizable distributions, where $\theta$ is given by the number $N_{\text{SV}}(\supersam)$ of support vectors in the supersample $\supersam$ \citep{mohri2018foundations}. Therefore, assuming $\Dist$ is a realizable distribution with respect to the class of half-spaces in $\Reals^d$, 
 the \loecmi of the SVM satisfies
\begin{equation*}
 \lecmi \leq \frac{\EE[N_{\text{SV}}(\supersam)]}{n+1}(2\log(n+1)+1).
\end{equation*}
Combining this result with \cref{thm:genbound-consistent} yields a bound on the risk of SVM that is optimal up to a constant factor \citep{long2020complexity}.

\subsection{Universality of \loecmi}

\label{sec:universality}
In this section, we demonstrate that leave-one-out CMI captures the asymptotics of risk for consistent learners.  More precisely,  for every data distribution  $\Dist$ and consistent learner $\TheAlg$, the quantity $\lecmi/ \log(n+1)$ vanishes as the number training samples $n$ diverges if and only if the risk also vanishes.
Also, for a broad class of learning algorithms for which the population risk converges to zero polynomially in the size of the training set,  $\lecmi/ \log(n+1)$ vanishes at the same rate as $\Risk{\Dist}{\Alg}$. 
\begin{theorem}
\label{thm:univ}
Let $\Alg$ be an interpolating learning algorithm for some data distribution $\Dist$. 
Assume loss lies in $\{0,1\}$. Then, 
\[
\label{univ-upperbound}
\lecmi \leq \binaryentr{\Risk{\Dist}{\Alg}} + \Risk{\Dist}{\Alg} \log (n+1)
\]
and
\[
\label{univ-lowerbound}
\Risk{\Dist}{\Alg} \log(n+1) \leq \lecmi.
\]
\end{theorem}
\begin{proof}
Since the binary entropy function is concave, it follows from Jensen's inequality, \cref{thm:leave-one-out}, and the identity 
$\EE[\loer]=\Risk{\Dist}{\Alg}$ that
\*[
\entr{L\vert \supersam} &= \EE[\centr{\supersam}{L}]\\
				&\leq 	\EE\big[ \binaryentr{\loer}+ \loer \log(n+1)\big]\\
				&\leq \binaryentr{\Risk{\Dist}{\Alg}} + \Risk{\Dist}{\Alg} \log (n+1),
\]
which was to be shown. Finally, the lower bound (\cref{univ-lowerbound}) is \cref{thm:genbound-consistent}.
\end{proof}
\begin{remark}
\label{rem:poly-decay}
Are these bounds tight for consistent learners in the large $n$ limit? 
First consider the case that the  risk of $\Alg$ does not converge to zero as $n$ diverges, i.e.,  $\Risk{\Dist}{\Alg}=\Theta(1)$. 
Then $\lecmi/\log(n+1)=\Theta(\Risk{\Dist}{\Alg})$. 
For consistent learning algorithms such that $\Risk{\Dist}{\Alg} = c \frac{\log(n)^\alpha}{n^\beta} $ 
where $c$, $\alpha$ and $\beta$ 
are some non-negative constants, 
we claim that $\lecmi/\log(n+1)=\Theta(\Risk{\Dist}{\Alg})$. This claim follows by bounding \cref{univ-upperbound} using the well-known inequality $\binaryentr{p}\leq -p\log(p)+p$ for $p \in [0,1]$.
\end{remark}
\begin{remark}
\label{rem:seperation}
Let $\inspace=[0,1]$ and $\outspace=\{0,1\}$.
The class of (right-continuous) one-dimensional thresholds over $\inspace$ is $\HS=\{h_\theta\vert \theta\in \inspace\}$ where $h_\theta(x)=\mathds{1}[x> \theta]$. Consider a continuous realizable distribution $D$ for $\HS$ with positive margin, i.e., the data labelled 0 and 1 are separated by an interval. 
For any proper algorithm, we may write $\TheAlg(\SS) = h_{\hat \theta}$ for some random variable $\hat{\theta}$ in $\inspace$. Using the margin assumption, we can design $\TheAlg$ so that (1) $h_{\hat{\theta}}$ achieves zero training error yet (2) the representation of $\hat{\theta}$ encodes the whole training set, in the sense that, having $\hat{\theta}$ and $\supersam$, we can decode $U$ perfectly. For this algorithm, we then have $\minf{\TheAlg(\SS);U\vert \supersam}=\log(n+1)$. On the other hand, for this class, \citep{hanneke2016refined} showed that any algorithm with zero training error achieves $\Risk{\Dist}{\Alg}=O(1/n)$. Therefore, our result in \cref{thm:univ} shows that $\minf{L;U\vert \supersam}/\log(n+1)=O(1/n)$. This example separates  $\minf{L;U\vert \supersam}$ and $\minf{\TheAlg(\SS);U\vert \supersam}$, and served as motivation for \cref{def:loocmi}. As we discuss elsewhere, if sufficiently tight control on generalization can be obtained via a formally looser bound, doing so yields am stronger explanation. 
\end{remark}
\section{An Optimal Bound for Learning VC classes}
\label{sec:appl}
We start this section with some standard definitions  \citep{shalev2014understanding}. Consider binary classification, i.e., $\outspace=\{0,1\}$, with zero--one loss. 
A sequence $((x_1,y_1),\dots,(x_n,y_n))$ is said to be \defn{realizable by $\HS$}, if for some $h \in \HS$, $h(x_i)=y_i$ for all $i \in \range{n} = \{1,\dots,n\}$. Note that, if $\Dist$ is realizable by $\HS$, then, for all $n \in \Naturals$, the sequence $\SS \sim \Dist^n$ is a.s.\ realizable by $\HS$. We say $\HS$ \emph{shatters} $(x_1,\dots,x_m)\in \inspace^m$ if for all $(y_1,\dots,y_m)\in \{0,1\}^m$, there exists $h \in \HS$, such that, for all $i\in \range{m}$, we have $h(x_i)=y_i$. 
The \emph{VC dimension} of $\HS$, denoted $\vcd$, is the supremum of integers $m \ge 0$ for which there exists $(x_1,\dots,x_{m})\in \inspace^{m}$ shattered by $\HS$. In particular, if there is no largest such integer, the VC dimension is infinite, i.e., $\vcd = \infty$.

In this section, we study the \loecmi of the classical one-inclusion graph algorithm, which was first proposed  by \citet{haussler1994predicting} as a general-purpose transductive learner for VC classes in the realizable setting. Here, we provide a brief description of this algorithm. Let $\HS$ be a class with a bounded VC dimension. Assume a realizable sequence of $n$ labeled samples $\SS=((x_1,y_1),\dots,(x_n,y_n))\in \dataspace^n$ and a test point $x_{n+1}$ are given to the learner, and the learner is tasked to predict the label of $x_{n+1}$. 
Let $V$ be the set of all possible labelings of $(x_1,\dots,x_{n+1})$ by the classifiers in $\HS$, i.e., 
$V=\{(v_1,\dots,v_{n+1}) \in \{0,1\}^{n+1} \mid \exists h \in \HS,\ \forall i\in [n+1], h(x_i)=v_i  \}$.
The \emph{one-inclusion graph} \citep{haussler1994predicting}
is the graph with vertex set $V$ such that vertices $\vec{v},\vec{w} \in V$ are connected by an edge if and only if the Hamming distance between $\vec{v}$ and $\vec{w}$ is one.
A \emph{probability assignment} is a function $P: V \times V \to [0,1]$ such that (i) $P(\vec{v},\vec{w}) > 0$  only if $\vec{v}$ and $\vec{w}$ are adjacent 
(in particular, $P(\vec{v},\vec{v})=0$ for all $\vec{v} \in V$) and (ii) given two adjacent vertices $\vec{v}$ and $\vec{w}$, we have $P(\vec{v},\vec{w}) + P(\vec{w},\vec{v})=1$.
We assume that $P$ is chosen based only on $(x_1,\dots,x_{n+1})$, i.e., independently of the labels.

We say a vertex $\vec{v}=(v_1,\dots,v_{n+1}) \in V$ is consistent with the labels in $\SS$ if $(v_1,\dots,v_n)=(y_1,\dots,y_n)$. Since the labels of $(x_1,\dots,x_n)$ are known, \citet{haussler1994predicting} observed that at most \emph{two} vertices in the one-inclusion graph are consistent with the labels in $\SS$. 
In the case that only vertex $\vec{v}=(v_1,\dots,v_{n+1}) \in V$ is consistent with $\SS$, the label of $x_{n+1}$ is predicted as $v_{n+1}$. 
In the case that two vertices $\vec{v}=(v_1,\dots,v_{n+1})$ and $\vec{w}=(w_1,\dots,w_{n+1})$ are consistent with $\SS$, they differ only on the $(n+1)$-th position and the algorithm uses the probability assignment $P$ to predict that the label for $x_{n+1}$ agrees with $v_{n+1}$ with probability $P(\vec{v},\vec{w})$ and agrees with $w_{n+1}$ otherwise. 

Consider a realizable distribution $\Dist$ and let $\targetfun \in \HS$ denote a function that determines the labels. \citet{haussler1994predicting} prove that the leave-one-out error of the one-inclusion graph algorithm can be expressed in terms of the probability assignment $P$ as
\[
\nonumber
\loer = \frac{\sum_{\vec{v} \in V}{P(\vec{v^\star},\vec{v})}}{(n+1)},
\]
where
$\vec{v^\star}=(\targetfun(X_1),\dots,\targetfun(X_{n+1})) \in V$ is the vertex corresponding to $\targetfun$ in the one-inclusion graph of $(X_1,\dots,X_{n+1})$. 
Moreover, \citet{haussler1994predicting} prove that there exists a probability assignment $P$ such that $\sum_{\vec{v}\in V}{P(\vec{w},\vec{v})}\leq \vcd$ uniformly over all $\vec{w} \in V$. By combining this result and \cref{thm:one-inclusion-bound}, we obtain the main result of this section.

\begin{theorem}
\label{thm:one-inclusion-bound}
Let $\TheAlg$ denote the one-inclusion graph algorithm. Then, for every VC class $\HS$ with dimension $\vcd$, every data distribution $\Dist$ realizable by $\HS$, and $n\geq \vcd$, there exists a probability assignment for the one-inclusion graph algorithm such that   
$\lecmi \leq \frac{\vcd}{n+1}(2\log(n+1)+1)$. Combining this results with the generalization bound in \cref{thm:genbound-consistent} yields expected risk  $\frac{2\vcd}{n+1}(1+o(1))$  which is optimal up to a constant factor \citep{li2001one}. 
\end{theorem}

Using our general reduction (\cref{thm:leave-one-out}), we have characterized the \loecmi of the one-inclusion graph algorithm and, as a consequence, shown that the \loecmi framework provides an \emph{optimal} bound for learning VC classes. It is worth noting that the IOMI framework of \citet{RussoZou16,XuRaginsky2017} is provably \emph{unable} to characterize the learnablity of VC classes \citep{roishay}, and, at present,
the best known bound within the CMI framework \citep{steinke2020reasoning}  
is 
\emph{suboptimal} by a $\log(n)$ factor \citep{steinke2020openproblem}.

 \begin{figure}
    \centering
    \includegraphics[scale=0.50]{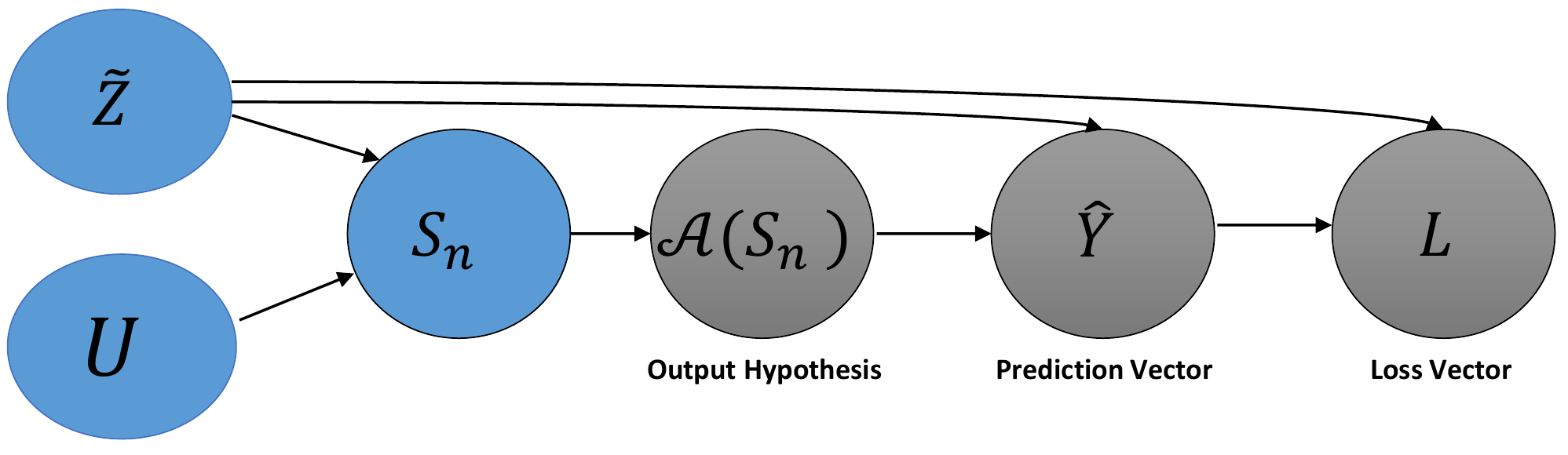}
    \caption{Conditional independence relationships encoded as a graphical model.}
    \label{fig:learning-problem-RV}
\end{figure}

\section{A Hierarchy of Measures of Information }
\label{sec:grun}
We have introduced leave-one-out CMI and shown that it can be used to bound the expected generalization error of learning algorithms. 
In this section, we aim to relate leave-one-out CMI to other measures of mutual information, 
some of which have already been shown to control generalization. 

To begin, we place leave-one-out CMI in a chain of inequalities:
\[
\begin{split}
\minf{L;U} 
&\stackrel{(a)}{\leq} \minf{L;U\vert \supersam} \stackrel{(b)}{\leq} \minf{\hat{Y};U\vert \supersam} 
\stackrel{(c)}{\leq} \minf{\Alg(\SS);U\vert \supersam} \\
&\stackrel{(d)}{\leq} \minf{\Alg(\SS);\SS}
\end{split}
\label{eq:diff-depmsr}.
\]

\cref{fig:learning-problem-RV} presents the conditional independence relationsihps that hold among the various random variables we have introduced. 
In the chain of inequalities, $(a)$ follows from the fact that $U \indep \supersam$; $(b),(c)$ follow from the data-processing inequality; and $(d)$ is shown by \citet{haghifam2020sharpened}. 
Except for the first quantity, $\minf{L;U}$, 
each quantity, or some close analogue, has been studied in the literature. For example, $\minf{\hat{Y};U\vert \supersam}$ was studied by \citet{harutyunyan2021information} in the context of CMI with a supersample of size $2n$. \NA{Indeed, the same sequence of inequalities hold for the analogous $2n$-supersample quantities in the CMI framework.}

We now show that $\minf{L;U}$ exactly determines the risk (equivalently, expected generalization error) of consistent algorithms for $0$--$1$ loss. For bounded loss functions, $\minf{L;U}$ upper-bounds the generalization error of arbitrary learning algorithms.
\begin{theorem}
\label{thm:tat}
Let $\Dist$ be a  distribution on $\dataspace$ and let $\Alg$ be a learning algorithm for $n \in \Naturals$ i.i.d.\ data.
For any $[0,1]$-bounded loss function,
\*[
\EGE \leq \sqrt{2 \minf{L;U}}.
\]
If $\Alg$ is almost surely interpolating,
then, for zero--one loss, 
\*[
\Risk{\Dist}{\Alg} = \frac{\minf{L;U}}{\log(n+1)}. 
\]
\end{theorem}
\begin{proof}
Consider a $[0,1]$-bounded loss.
For all $i \in [n+1]$, 
let $\rho_i:[n+1]\to \Reals$ be
\begin{equation*}
 \rho_i(j)= 
\begin{cases} 
1 &\mbox{if } j = i, \\
-\frac{1}{n} & \mbox{otherwise. }  \end{cases}   
\end{equation*}

Let $\Dist$, $\supersam$, $U$,  $L$, and $\tilde{U}$ be as in the proof of \cref{thm:genbound-consistent}. 
By the Donsker--Varadhan variational formula \citep[Prop.~4.15]{boucheron2013concentration}, for all bounded measurable functions $h$ and for all $\lambda\in \Reals$
\[
\begin{split}
\label{eq:dv-lemma1}
&\minf{L;U} = \KL{ \Pr{(U,L})}{\Pr{({L})}\otimes\Pr{(\tilde{U})}} \\
&\geq \Pr{(U,L)}(\lambda h) - \log \big[\big(\Pr{({L})}\otimes\Pr{(\tilde{U})}\big)(\exp (\lambda h))\big].
\end{split}
\]
Consider now the function $h: [n+1] \times [0,1]^{n+1} \to [-1,1]$  given by $h(u,l) = \sum_{i=1}^{n+1}\rho_i(u) l_i$.
Then
\[
&\Pr{(U,L)}(h) = \EE[ \cEE{U}{[h(U,L)]}] \nonumber\\
&=\EE\big[\cEE{U}{\big[ \ell(\Alg(\supersam_{-U}),\supersam_U) - \frac{1}{n}\sum_{i\in[n],i\neq U} \ell(\Alg(\supersam_{-U}),\supersam_i)\big]}\big] \nonumber\\
&= \EE \big[ \Risk{\Dist}{\Alg(\SS)}-\EmpRisk{\SS}{\Alg(\SS)} \big] \nonumber\\
& = \EGE .
\label{eq:dv-egeterm}
\]
Moreover, for all $i \in [n+1]$, we have $\EE[\rho_i(\tilde{U})]=0$. Therefore $\cEE{L}{h(\tilde{U},L)}=0$ since  $\tilde{U} \indep L$. Using Hoeffding's lemma and the fact that $|h|\leq 1$, we obtain
\[
\Pr{(L)}\otimes\Pr{(\tilde{U})}(\exp (\lambda h)) &= \EE\big[ \cEE{L}{\exp (\lambda h(\tilde{U},L))} \big] 
\leq \exp(\lambda^2/2) \label{eq:hoeffeding-lemma}. 
\]
By \cref{eq:dv-lemma1,eq:dv-egeterm,eq:hoeffeding-lemma}, 
 $\EGE \leq \frac{\minf{L;U}}{\lambda} + \frac{\lambda}{2} \label{eq:prefinal-agnostic}$.
Finally, letting $\lambda = \sqrt{2 \minf{L;U}}$, we obtain the stated result.

Now consider 0--1 loss and assume $\Alg$ is interpolating.
By the definition of the mutual information, 
we have $\minf{L;U} = \entr{L} - \entr{L\vert U}$. 
Let $\lossreal{i}$, $i \in \{0,\dots,n+1\}$ be as in the proof of \cref{thm:univ}. 
Because $\Alg$ is interpolating, the support of $\lossvec$ is $\{\lossreal{i}\vert i \in \{0,\dots,n+1\} \}$. For $i >0$, %
\*[
\Pr(L=\lossreal{i}) = \frac{1}{n+1}\Pr(\loss(\Alg(\supersam_{-i}),\supersam_i)=1) = \frac{\Risk{\Dist}{\Alg}}{n+1},
\]
where we have used the fact that
$\Risk{\Dist}{\Alg}=\Pr(\loss(\Alg(\supersam_{-i}),\supersam_i)=1)$
for $i > 0$.
Therefore, %
\[
\entr{L}&= -\sum_{i=0}^{n+1} \Pr(L=\lossreal{i}) \log(\Pr(L=\lossreal{i}))\nonumber\\
&=-(1-\Risk{\Dist}{\Alg}) \log(1-\Risk{\Dist}{\Alg}) - \Risk{\Dist}{\Alg} \log \frac{\Risk{\Dist}{\Alg}}{n+1} \label{eq:tat-term1}.
\]
Similarly, we have
\[
\entr{L\vert U} &= \frac{1}{n+1} \sum_{i=1}^{n+1} \centr{U=i}{L} \nonumber\\
&= -(1-\Risk{\Dist}{\Alg}) \log(1-\Risk{\Dist}{\Alg}) - \Risk{\Dist}{\Alg} \log \Risk{\Dist}{\Alg}\label{eq:tat-term2}.
\]
Using \cref{eq:tat-term1}, \cref{eq:tat-term2}, and the definition of the mutual information, the stated result follows. 
\end{proof}

\begin{remark}[Proof of \cref{thm:gen-agnostic}]
From \cref{thm:tat} and the inequality 
$\minf{L;U} \leq  \minf{L;U\vert \supersam}$,
we obtain \cref{thm:gen-agnostic}.
\end{remark}

 \begin{remark}[Maximal gaps]
Using the fact that $\minf{\Alg(\SS);U\vert \supersam} \le \entr{U} \le \log (n+1)$,
 \cref{thm:tat} and \cref{eq:diff-depmsr} imply that, for 0--1 loss and interpolating learning algorithms,
 we have
\begin{equation*}
   \Risk{\Dist}{\Alg} = \frac{\minf{L;U}}{\log(n+1)} \leq  \frac{\minf{L;U\vert \supersam}}{\log(n+1)} \leq  \frac{\minf{\hat{Y};U\vert \supersam}}{\log(n+1)} \leq  \frac{\minf{\Alg(\SS);U\vert \supersam}}{\log(n+1)} \le 1.  
\end{equation*}
Starting from this chain of inequalities, we can 
investigate the gap between these different measures of dependency. 
Of course, the maximal (additive) gap between all these measures is $\log(n+1)$. Is this gap achieved?
\begin{itemize}[leftmargin=1.5em]
    \item  $\minf{\hat{Y};U\vert \supersam}$ versus $\minf{\Alg(\SS);U\vert \supersam}$: 
    Note that, for binary classification under 0--1 loss, we have $\minf{L;U\vert \supersam}=\minf{\hat{Y};U\vert \supersam}$. This is due to $L$ being a one-to-one function of $\supersam$ and $\hat{Y}$.
    As shown in \cref{rem:seperation}, there is a maximal gap between $\minf{L;U\vert \supersam}$ and $\minf{\Alg(\SS);U\vert \supersam}$.  Therefore,  there exists a learning scenario in which  $\minf{\Alg(\SS);U\vert \supersam}=\Omega(\log(n+1))$ while  $\minf{\hat{Y};U\vert \supersam} \in o(\log(n+1))$. 

    \item  $\minf{L;U\vert \supersam}$ versus $\minf{\hat{Y};U\vert \supersam}$: 
    Consider a setting where $\inspace=\outspace=[-1,1]$ and the loss function is $\loss(\hat y,(x,y)) = \indic{\hat y y<0}$. 
    Let the data distribution be that of $(X,\targetfun(X))$, where $X$ is uniformly distributed on $\inspace$ and $\targetfun(x)=\indic{x>0}$. Consider the learning rule $\Alg(\SS)(x) = y$ where $(x,y) \in \SS$ and $\Alg(\SS)(x)=x$ otherwise. This algorithm is consistent by design and the expected risk of this algorithm is zero. Therefore, \cref{thm:univ} and \cref{thm:tat} show that $\minf{L;U}=\minf{L;U\vert \supersam}=0$. However, it can be easily seen that $\minf{\hat{Y};U\vert \supersam}=\log(n+1)$. Thus there exists a learning scenario such that there is a maximal gap between $\minf{\hat{Y};U\vert \supersam}$ and $\minf{L;U\vert \supersam}$.

    \item  $\minf{L;U}$ versus $\minf{L;U\vert \supersam}$:  
    Our result in \cref{thm:univ} shows that the gap between these two quantities cannot be maximal. Also, as mentioned in \cref{rem:poly-decay}, for learning algorithms whose expected risk decays polynomially in $n$ to zero, $\minf{L;U}$ can only be tighter than $\minf{L;U\vert \supersam}$ by a constant factor. Note that since $\minf{L;U} \log(n+1)$ is the risk, any characterization of the gap between  $\minf{L;U}$ and $\minf{L;U\vert \supersam}$ is the same as characterization of the gap between $\minf{L;U\vert \supersam}$ and the risk. Understanding the exact gap between these two measures is important future work.
\end{itemize}
 \end{remark}
 
\begin{remark}
One advantage
of working with $\minf{L;U\vert \supersam}$ and $\minf{\hat{Y};U\vert \supersam}$, over
non-evaluated \loecmi $\minf{\Alg(\SS);U\vert \supersam}$,
is that 
the former quantities do not require one to have a  parametrization of the set of possible classifiers.
Of course, the training data themselves always serve as a  ``parametrization'' for nonrandomized learning rules, but such a parameterization leads to vacuous bounds.
(The same roadblocks pertain to plain CMI.)
 A quintessential example of a setting where natural parametrization may not exist is that of transductive learning algorithms, i.e., ones whose input includes the test input and whose output is the corresponding label prediction. The $k$-Nearest Neighbor Algorithm is one specific example. 
 \end{remark}

 \begin{remark}
  The leave-one-out CMI framework provides numerous estimates of the expected generalization error based on various measures of information. Which measure should one use to understand generalization?  
  To simplify the discussion, let us focus on the case of interpolating learners under 0--1 loss, in which case expected generalization error is simply risk.
  \cref{thm:tat} indicates that $I(L;U)$ is equal to the risk. As such, there appears to be no advantage to studying $I(L;U)$.  Of course, once one recognizes that risk is a mutual information, one can invoke information-theoretic results to obtain quantities that may be easier to estimate, such as provided by the bound $\minf{L;U} \leq  \minf{L;U\vert \supersam}$.

  Even if one can directly bound or compute $\minf{L;U}$, there may be advantages to studying measures that are never tighter, such as the quantities later in the chain of inequalities in \cref{eq:diff-depmsr}. 
  As argued by \citet{dziugaite2020robustmeasures}, 
  if a formally looser quantity controls generalization error (or risk) to a sufficient extent,
  then this looser quantity provides a more general explanation of the empirical generalization phenomena, as the adequacy of any formally tighter bound is then tautological. That is, when attempting to explain generalization phenomena, use the loosest bound that suffices. 
  
  This perspective also suggests that identifying tighter bounds should not be the goal of 
   studying generalization from an information-theoretic perspective. Instead, we seek a rich hierarchy of bounds and an understanding of their interrelationships, so that we can come to understand generalization in specific instances in terms of the level in this hierarchy needed to explain the phenomenon.  

  Besides the challenge of identifying the right quantity to explain a phenomenon of interest, there are statistical and computational barriers to studying generalization. The measures of information presented here and studied by other authors all depend on the data distribution, which in many interesting settings is not known, other than through a random sample. Even in cases where certain distributions are known, many of these quantities are computationally intractable, without exploiting special structure. There are ways to navigate around these roadblocks. One example is demonstrated by work using ``data-dependent estimates'' to study generalization in iterative algorithms in deep learning \citep{negrea2019information,haghifam2020sharpened,wang2021analyzing}.
 \end{remark}

\section{Conclusion}
We have presented a leave-one-out variant of the CMI framework,  a novel information-theoretic framework to reason about generalization in machine learning.  
For 0--1 loss and interpolating learning algorithms, \loecmi provides upper and lower bounds on risk. For consistent learners that are interpolating for any number of data, 
the \loecmi framework captures the asymptotics of risk when the risk converges to a nonzero quantity or to zero polynomially.
As an application of the \loecmi framework, we have studied the leave-one-out CMI of the one-inclusion graph algorithm \citep{haussler1994predicting}, and shown that the framework yields an optimal risk bound for learning VC classes in the realizable setting. At present, it is not known whether optimal bounds for this setting can be achieved via any other existing information-theoretic framework.

\textbf{Open problems.} Our work raises several open problems: 

\begin{enumerate} 
\item 
Is
\cref{thm:gen-agnostic} tight? 
In what settings (outside those of  \cref{thm:univ}) can one obtain tighter bounds?
In particular, can one use the leave-one-out CMI framework to obtain tight (up to universal constants) bounds on the generalization error of arbitrary algorithms, similar to what we showed for interpolating learning algorithms?
\item
Is \cref{thm:univ} tight for finite $n$? Under what conditions can we remove or tighten the binary entropy term?
\item
Can one obtain an optimal bound for the one-inclusion graph algorithm under realizability via the standard CMI framework, or is there a lower bound? Is there any optimal (improper) learner for VC classes under realizability for which CMI yields optimal bounds?
Lower bounds here would demonstrate that leave-one-out CMI is fundamentally stronger.
\item
Leave-one-out CMI can be interpreted as an information-theoretic notion of stability. Are there connections to notions of algorithmic stability  \citep{elisseeff2003leave}? 
\end{enumerate}

In general, it is an open challenge to determine the  \loecmi of common learning algorithms to better understand this framework.

\section*{Acknowledgments}

MH is supported by the Vector Institute.  DMR is supported in part by Canada CIFAR AI Chair funding through the Vector Institute, an NSERC Discovery Grant, Ontario Early Researcher Award, a stipend provided by the Charles Simonyi Endowment, and a New Frontiers in Research Exploration Grant.
SM is a Robert J.\ Shillman Fellow, his research is supported in part by 
the Israel Science Foundation (grant No.\ 1225/20), by a grant from the United States - Israel
Binational Science Foundation (BSF), by an Azrieli Faculty Fellowship, by Israel PBC-VATAT, 
and by the Technion Center for Machine Learning and Intelligent Systems (MLIS).
This work was done in part while GKD and DMR were visiting the Simons Institute for the Theory of Computing.

\renewcommand{\bibfont}{\small}
\printbibliography

\end{document}